\documentclass{article}

\usepackage{microtype}
\usepackage{graphicx}
\usepackage{subcaption}
\usepackage{booktabs}
\usepackage{enumitem}

\usepackage{threeparttable}

\usepackage{hyperref}

\usepackage[preprint]{icml2026}

\usepackage{amsmath}
\usepackage{amssymb}
\usepackage{mathtools}
\usepackage{amsthm}

\usepackage[capitalize,noabbrev]{cleveref}

\theoremstyle{plain}
\newtheorem{theorem}{Theorem}
\newtheorem*{formalmain}{Theorem~\ref{thm:main}}
\newtheorem{proposition}[theorem]{Proposition}
\newtheorem{lemma}[theorem]{Lemma}

\theoremstyle{definition}
\newtheorem{definition}[theorem]{Definition}

\theoremstyle{remark}
\newtheorem{remark}[theorem]{Remark}

\begin{document}

\twocolumn[
  \icmltitle{Minimum Neural Weight Norm = Kolmogorov Complexity}

  \begin{icmlauthorlist}
    \icmlauthor{Tiberiu Mușat}{eth}
  \end{icmlauthorlist}

  \icmlaffiliation{eth}{ETH Zurich}

  \icmlcorrespondingauthor{}{tiberiu@musat.ai}

  \icmlkeywords{Machine Learning, ICML}

  \vskip 0.3in
]

\printAffiliationsAndNotice{}

\begin{abstract}
  Why do neural networks generalize better when their parameter norm is small? We prove that, in looped neural networks with constant-precision weights, the smallest weight norm of a network computing a function is asymptotically equivalent to the Kolmogorov complexity of that function, divided by its logarithm. The proof has two reductions: looped neural networks can encode information at a rate of logarithmically many bits per non-zero parameter encoded via the sparsity pattern, while any fixed-precision network can be described by enumerating its non-zero parameters with logarithmic addressing cost. Both bounds have the same leading constant factors, giving an asymptotic equivalence. More broadly, by establishing a theoretical link between norm minimization and algorithmic simplicity, we provide an avenue towards understanding generalization in modern deep learning.
\end{abstract}

\newcommand{\defeq}{\vcentcolon=}
\newcommand{\eqdef}{=\vcentcolon}

\newcommand{\precision}{\delta}
\newcommand{\maxVal}{M}
\newcommand{\RR}{\mathbb{R}}
\newcommand{\ZZ}{\mathbb{Z}}
\newcommand{\NN}{\mathbb{N}}
\newcommand{\weightSet}{\mathcal{Q}}

\newcommand{\str}{s}
\newcommand{\alphabet}{\Sigma}
\newcommand{\alphabetExt}{\Tilde{\Sigma}}
\newcommand{\BOS}{\mathtt{BOS}}
\newcommand{\EOS}{\mathtt{EOS}}
\newcommand{\EXT}{\mathtt{EXT}}
\newcommand{\embedDim}{d_{\mathrm{e}}}
\newcommand{\embed}{\mathbf{E}}
\newcommand{\embedSpace}{\mathcal{Q}^{\embedDim \times |\alphabetExt| }}
\newcommand{\x}[1]{\mathbf{X}^{(#1)}}
\newcommand{\seqLen}[1]{N_{#1}}

\newcommand{\arch}{\alpha}
\newcommand{\archSet}{\mathcal{A}}
\newcommand{\param}{\theta}
\newcommand{\paramSet}[1]{\weightSet^{\size{#1}}}
\newcommand{\size}[1]{d(#1)}
\newcommand{\netSet}{\Theta}
\newcommand{\pair}{{\arch, \param}}
\newcommand{\net}{{(\pair)}}

\newcommand{\loss}{\mathcal{L}}

\section{Introduction}

Neural networks learn better when we force their weights to be small, usually through a decay term applied after every update step. This technique, commonly known as \emph{weight decay} \citep{krogh1991simple}, has become an essential ingredient of the modern deep learning recipe \citep{loshchilov2019decoupled}. Moreover, several implicit biases of deep learning are known to encourage small weight norms \citep{nacson2019stochastic,NIPS2013_38db3aed,lyu2020gradient}. However, despite its ubiquity, we have relatively little insight into why norm minimization improves generalization almost universally.

Statistical generalization guarantees have traditionally been based on representational capacity. This becomes unsatisfactory for modern overparametrized models, which can memorize large amounts of random noise \citep{zhang2017understanding}. In fact, once the number of model parameters is greater than the number of training samples, further increasing the parameter count improves generalization, a phenomenon known as \emph{double descent} \citep{belkin2019reconciling}. On the other hand, weight decay alone can induce generalization even after the training data has been fully memorized, without introducing any novel training samples, a phenomenon known as \emph{grokking} \citep{power2022grokking, liu2022omnigrok}.

How can we explain the remarkable power of the prior induced by norm minimization? We propose an answer to this question based on \emph{algorithmic information theory (AIT)}. The principle of `simplicity', also known as Occam's Razor, states that the simplest explanation is the most likely to be true. This was formalized in the theory of universal induction \citep{solomonoff1964formal, levin1973universal}, which argues that the shortest program is the most likely to generalize. We prove an asymptotic equivalence between (1) the smallest weight norm of a neural net, required to compute a certain function, and (2) the length of the shortest program computing the same function, divided by its logarithm.

We focus on \emph{looped} neural networks with parameters stored in \emph{constant precision}. Looping is a necessary ingredient to make neural networks Turing complete \citep{weiss2018practical}. Constant precision is motivated by both practical and theoretical considerations. On the practical side, neural networks are trained with constant precision due to hardware and efficiency constraints \citep{micikevicius2018mixed}. On the theoretical side, neural networks with real-valued weights are known to be strictly more powerful than the Turing model \citep{siegelmann1995computation}. Moreover, there is some evidence that biological neural networks encode synaptic strength with (relatively low) constant precision \citep{bartol2015nanoconnectomic}.

Our main result is the following:
\begin{theorem}
\label{thm:main}
    Let $\alphabet$ be a fixed finite input/output alphabet.  Then, for all total computable functions $f : \alphabet^* \rightarrow \alphabet^*$,
    \begin{align}
        \min_{F_\param = f}
        \frac{\|\param\|_p^p}{\precision^p}
        \;\;\sim\;\;  \frac{K(f)}{\log_2 K(f)},
        &&
        (K(f) \to \infty),
    \end{align}
    where $K(f)$ denotes the Kolmogorov complexity of $f$, that is, the length of the shortest program computing it on a fixed optimal prefix-free machine, and where $\param$ is the parameter vector of a Constant-Precision Looped Transformer (CPLT) with precision $\precision$. We use $F_\param$ to denote the function realized by a CPLT with weights $\param$. The equivalence $a \sim b$ means that $a/b \to 1$.
\end{theorem}

The logarithmic factor comes entirely from the sparsity pattern, since constant-precision values store $O(1)$ bits of information. On one hand, simply enumerating the non-zero weights incurs a logarithmic addressing cost. On the other hand, we show how to construct sparsity patterns that exploit the logarithmic overhead tightly, resulting in sharp sandwich bounds. Therefore, our main result can also be seen as a tight bound on the information that can be encoded in the sparsity patterns of looped neural networks.

Our main result also implies that learning with weight decay is a universal compressor in the following sense: the global optimum of the standard training objective offers approximately as much compression as the maximum a posteriori (MAP) estimate under the universal prior $\mathbb{P}(f) \propto 2^{-K(f)}$ of \citet{solomonoff1964formal}, assuming a Turing-universal architecture and a well-calibrated weight decay penalty coefficient. We formalize this connection in \cref{sec:compression}.

However, this connection with the universal prior holds only for the global optima. Regarding the full posterior distributions over weights under common priors (e.g. Gaussian), they appear unlikely to approximate the posterior over functions under the universal prior, according to our preliminary investigations in \cref{sec:priors}.

\paragraph{Paper Overview.} We review related work in \cref{sec:related-work}.
Our results are stated for a transformer architecture, described in \cref{sec:cplt}, but the equivalence generally holds in any constant-precision looped neural net with access to an unbounded memory tape. The tight asymptotic bound follows from two reductions. Reducing a transformer to a program, as described in \cref{sec:cplt-to-program}, gives the lower bound. Reducing a program to a transformer, as described in \cref{sec:program-to-cplt}, gives the upper bound. In \cref{sec:generalization-bounds}, we derive generalization bounds in terms of the network weight norm by applying the theory of minimum description length. In \cref{sec:compression}, we formalize the implied connection of norm minimization and universal compression. In \cref{sec:priors}, we explore the connection between priors over neural network weights and priors over programs. Finally, we discuss our results in \cref{sec:discussion}.

\section{Related Work}
\label{sec:related-work}

\subsection{Norm Minimization}

Norm minimization as a form of regularization has roots in ill-conditioned least squares problems \citep{tikhonov1963solution, hoerl1970ridge}. In deep neural networks, \citet{krogh1991simple} formalized the use of a loss penalty term proportional to the squared $L_2$ parameter norm, also equivalent to a linear decay term at every training step.
Despite some early attempts to improve generalization by reducing the number of parameters \citep{thodberg1991improving, weigend1990generalization}, norm minimization emerged as the more effective approach \citep{bartlett1996valid, neyshabur2015search} that remains essential today as compute scales \citep{kim2026pre}.

The squared $L^2$ penalty term can be understood using a Bayesian framework \cite{mackay1992practical} assuming a Gaussian prior over parameters. However, modern large nets are usually trained with a decoupled weight decay term \citep{loshchilov2019decoupled, liu2025muon}, breaking the classical Bayesian interpretation. Constraining the weight norm induces a well understood constraint on the set of learnable functions in deep linear networks \citep{dai2021representation}, as well as in shallow non-linear networks with infinite width \citep{savarese2019infinite, ongie2019function}.

More recently, weight decay has been observed to induce generalization even after the complete memorization of the training data \citep{power2022grokking, liu2022omnigrok}, a phenomenon termed \emph{grokking}.
A proposed theoretical framework for the underlying training dynamics is norm minimization on the memorization manifold \citep{boursier2026theoretical, musat2025geometry}.
In certain small synthetic tasks, such as modular addition, weight decay has been shown to induce the emergence of simple yet powerful representations \citep{nanda2023progress, morwani2024feature}.

\subsection{Occam's Razor in Deep Learning}

Several works argue that classical views are insufficient to explain the success of modern deep learning \citep{zhang2017understanding,ebtekar2026understanding}.
One promising solution comes from the principle of simplicity, also known as Occam's Razor \citep{sober2015ockham}, formalized by \citet{solomonoff1964formal} in his theory of universal induction. This theory defines a prior distribution over strings (technically referred to as the \emph{universal semimeasure}) using an underlying latent distribution over the programs that compute those strings, where the probability of a program decays exponentially in its length. Most notably, the universal semimeasure dominates every other lower-semicomputable semimeasure \citep{zvonkin1970complexity}.

Solomonoff induction could play a defining role in machine intelligence \citep{legg2007universal}. The main limitation of this theory is its dependence on a particular programming language \citep{herrmann2020pac, neth2023dilemma}, which could be chosen to fit any arbitrary measure of `simplicity' on finite data. \citet{ebtekar2026understanding} propose a possible solution: relativizing Solomonoff induction to an information vantage point, biasing toward short programs with access to all preexisting information.

\citet{jacot2025deep} proves an equivalence between binary circuit size and neural weight norm, which is very close in spirit to our contribution. Specifically, the main result of \citet{jacot2025deep} is a sandwich bound that relates, for any real function, the following two measures: (1) the minimum size of a binary circuit approximating it and (2) the minimum weight norm of a ResNet approximating it. \citet{mingard2025characterising} prove an equivalence between neural complexity and DNF complexity in 2-layer networks with ternary weights. \citet{ahvonen2025descriptive} study the complexity of floating-point neural networks, establishing a correspondence with Boolean programs.

\citet{shaw2026bridging} recently connected Kolmogorov complexity and Transformers by constructing asymptotically optimal codes. Leveraging the Turing universality of Transformers, it is possible to select a prior that assigns probability to sets of parameters based on the algorithmic complexity of the function they compute. They also propose a variational objective that may be optimized to find transformers of low algorithmic complexity. The search for simple neural networks has a long history \citep{hinton1993keeping, schmidhuber1997discovering}. Our work is quite different in spirit: we show that standard norm minimization is already a very strong prior for algorithmic simplicity. However, more efficient complexity minimizers remain of great interest.

\subsection{Looped Neural Nets}

Neural networks usually have bounded depth or memory, which means they are not Turing universal \citep{weiss2018practical, hahn2020theoretical}. Several neural designs have been proposed in order to achieve Turing completeness \citep{kaiser2016neural, graves2014neural}. These usually require applying the same neural module repeatedly in a loop, while ensuring access to an unbounded memory. In fact, even the standard Transformer architecture \citep{vaswani2017attention} becomes universal when looping the transformer layer and using a large enough context \citep{dehghani2018universal, giannou2023looped}. Notably, such looped transformers have proven very useful in practice \citep{lan2019albert, bae2024relaxed, zhu2026scaling}.

Another way to make transformers universal comes in a very different form of looping, but one that has been arguably even more successful in practice: allowing an autoregressive transformer to generate intermediate tokens, also known as `Chain-of-Thought' (CoT) or `scratchpad', provably makes it Turing complete \citep{perez2021attention, bhattamishra2020computational, merrill2024expressive, qiu2025ask}, even when weights are stored with constant precision \citep{li2025constant, jiang2025softmax}.

\subsection{Constant Precision}

While arbitrary precision is a common assumption in theoretical works \citep{perez2019turing, strobl2024what}, in practice neural networks are trained using fixed precision enforced by the hardware. In fact, despite models getting ever larger in recent times, efficiency constraints encourage even less precision for storing the weights during training \citep{micikevicius2018mixed, deepseek-ai2025deepseekv3}. One theoretical consideration makes infinite precision seem particularly unrealistic: real-valued weights can store an infinite amount of information, which makes them more powerful than the Turing model \citep{siegelmann1994analog, siegelmann1995computation}.

A very low precision could provide further efficiency gains, if the hardware is adapted accordingly. Motivated by this possibility, there is an extensive literature exploring neural nets where weights are stored with as few bits as possible, also known as \emph{quantization} \citep{gholami2021survey, lang2024comprehensive}. \citet{courbariaux2015binaryconnecta} showed that neural networks remain competitive even when their weights are constrained to only two possible values (e.g. -1 or 1). More recently, \citet{ma2024era} successfully trained transformer-based large language models with weights constrained to three possible values: -1, 0, or 1.

In neuroscience, a related line of work argues that biological neural networks also store information in quantized form \citep{mcculloch1943logical, vanrullen2003perception, tee2020information}. This could be theoretically justified by the fact that discrete representations are more robust to the noise present in physical environments \citep{chaudhuri2016computational, faisal2008noise}. There is experimental evidence for this view \citep{petersen1998allornone, oconnor_graded_2005}, with some estimates of information storage capacity at 4-5 bits per synapse \citep{bartol2015nanoconnectomic, samavat2024synaptic}.

\section{Constant-Precision Looped Transformer}
\label{sec:cplt}

We use a looped Transformer block \citep{dehghani2018universal} with constant precision weights and a growing context.

\subsection{Constant Precision}
All learned weights are represented with constant precision $\precision \in \RR$ and a maximum absolute value $\maxVal \in \RR$; thus the allowed weight values are
\begin{equation}
    \weightSet \;\defeq\;
    \weightSet_{\maxVal,\precision} \;\defeq\;
    \precision \ZZ \,\cap\, [-\maxVal, +\maxVal].
\end{equation}
We assume that $\maxVal \ge 1$ and $\precision^{-1} \in \NN$. For simplicity, we also assume that activations and other intermediate computations use rational arithmetic.

\subsection{Input}
The model takes an input string consisting of symbols from a predefined alphabet $\alphabet$. We assume that $|\alphabet| \ge2$. We use three special symbols to mark input boundaries and context extensions:
\begin{equation}
    \alphabetExt \defeq \alphabet \cup \{ \BOS, \EOS, \EXT \}
\end{equation}

We denote the neural activations after step $t$ as $\x{t}$, with $\x{t} \in \RR^{\embedDim \times \seqLen{t}}$, where $\embedDim$ denotes the embedding dimension and $\seqLen{t}$ denotes the context length at step $t$.

First, the input string $\str \in \Sigma^*$ is delimited and embedded
\begin{equation}
    \x{0} = \mathtt{Embed}\left( \BOS \;\str \; \EOS \right)
\end{equation}
using a learned embedding matrix $\embed \in \embedSpace$.

The initial context will have length $\seqLen{0} = |\str| + 2$. After each step, we extend the context to the right with one $\EXT$ symbol, resulting in length $\seqLen{t} = |\str| + 2 + t$.

\subsection{Transformer Block}
We repeatedly apply a Transformer block consisting of a multi-head attention layer and a multi-layer perceptron, then append an $\EXT$ symbol:
\begin{align}
    \mathbf{U}^{(t+1)} &=
\mathtt{MLP}\left(\mathtt{Attention}\left(\mathbf{X}^{(t)}\right)\right)
\\
    \mathbf{X}^{(t+1)} &= \big[ \;
    \mathbf{U}^{(t+1)} \;\; \mathtt{Embed}(\EXT) \; \big].
\end{align}

\newcommand{\numHeads}{n_h}
\newcommand{\headDim}{h}
\newcommand{\radius}{R}
\newcommand{\mlpDim}{d_{\mathrm{mlp}}}
\newcommand{\relu}{\mathrm{ReLU}}
We use global attention with local position relative biases around a constant radius $\radius \ge 1$ with $\numHeads$ heads, dimension $\headDim_k > 0$ for head $k \in [\numHeads]$, hard-max, no masking, residual connections, and relative positional biases:
\begin{align}
    &\mathtt{Attention}(\mathbf{X}) = \mathbf{X} + \sum_{k=1}^{\numHeads} \mathbf{O}^{(k)} \, \mathbf{V}^{(k)} \, \mathbf{X} \, \mathbf{A}^{(k)}
\end{align}
where
\begin{align}
    &\mathbf{A}^{(k)} =\mathrm{HardMax}\left(\mathbf{S}^{(k)}\right),
    \\
     &\mathbf{S}^{(k)}_{i,j} =
     \left( \mathbf{K}^{(k)} \, \mathbf{X}_{:,i} \right)
     ^\intercal
     \left( \mathbf{Q}^{(k)} \, \mathbf{X}_{:,j}\right)
     + \mathrm{pos}^{(k)}(i-j),
     \\
     &\mathrm{pos}^{(k)}(i-j) =
     \begin{cases}
        \mathbf{P}^{(k)}_{i-j+R} &\text{ if $|i-j| \le \radius$,}
        \\
        0 &\text{ if $|i - j| > \radius$}.
        \end{cases}
\end{align}
with learned weights $\mathbf{Q}^{(k)}, \mathbf{K}^{(k)}, \mathbf{V}^{(k)} \in \weightSet^{\headDim_k \times \embedDim}$,  $\mathbf{O}^{(k)} \in \weightSet^{\embedDim \times \headDim_k}$, and $\mathbf{P}^{(k)} \in \weightSet^{2\radius+1}$ for every head $k \in [\numHeads].$ $\mathrm{HardMax}$ acts column-wise with averaging on ties.

We use a standard MLP with a single hidden layer of size $\mlpDim$, $\relu$ non-linearity, and a residual connection:
\begin{equation}
    \mathtt{MLP}(\mathbf{X}) = \mathbf{X} + \mathbf{W}_\mathrm{out} \relu( \mathbf{W}_\mathrm{in} \mathbf{X} )
\end{equation}
where $\mathbf{W}_\mathrm{in} \in \weightSet^{\mlpDim \times \embedDim}$ and $\mathbf{W}_\mathrm{out} \in \weightSet^{\embedDim \times \mlpDim}$.

\subsection{Emitting Output \& Halting}

\newcommand{\wHalt}{\mathbf{w}_{\mathtt{halt}}}
\newcommand{\wEmit}{\mathbf{w}_{\mathtt{emit}}}
\newcommand{\unembed}{\mathbf{U}}

After every step, the model may emit a symbol from $\alphabet$ and decide whether to finish its execution. For that, we use two special neurons with weights $\wEmit, \wHalt \in \weightSet^{\embedDim}$, as well as an unembedding matrix $\unembed \in \weightSet^{|\alphabet| \times \embedDim}$. They are computed only at the first position on the context (i.e. where $\BOS$ is initially placed).

Denote $\mathbf{z}^{(t)} := \x{t}_{:, 0}$ and compute
\begin{align}
    \mathtt{emit}^{(t)} &=\wEmit^\intercal \mathbf{z}^{(t)}, &
    \mathtt{score}^{(t)} &=\unembed \mathbf{z}^{(t)}, \\
   \mathtt{halt}^{(t)} &=\wHalt^\intercal \mathbf{z}^{(t)}.
\end{align}
Emitting begins at $t \ge 1$. If $\mathtt{emit}^{(t)} > 0$, we emit the symbol $\mathtt{x} \in \alphabet$ with the largest $\mathtt{score}^{(t)}_{\mathtt{x}}$ (in case of equality, the first symbol is chosen). If $\mathtt{halt}^{(t)} > 0$, the model will not proceed to the next step and execution finishes.

The final output string will be composed of all emitted symbols, in the order that they were emitted, up to (and including) the step when halting occurred. If the network never halts, the emitted tokens are discarded and the output remains undefined.

\section{Formal Definition}
\label{sec:definitions}

We denote the vector of architectural parameters as
\begin{equation}
    \arch = (
        \embedDim,
        \mlpDim,
        \headDim_1, \ldots,\headDim_{\numHeads})
    \in \archSet,
\end{equation}
where $\archSet$ denotes the set of valid architectures:
\begin{equation}
    \archSet \;=\; \NN_{>0}^{\ge2}.
\end{equation}
We denote the vector of weights as
\begin{align}
    \param = \begin{bmatrix}
        \mathtt{flat}(\embed) \\
        \wHalt \\
        \wEmit \\
        \mathtt{flat}(\unembed) \\
        \mathtt{flat}(\mathbf{P}) \\
        \mathtt{flat}(\mathbf{W}_{\mathrm{in}}) \\
        \mathtt{flat}(\mathbf{W}_{\mathrm{out}}) \\
        \mathtt{flat}(\mathbf{K}^{(1)}) \\
        \mathtt{flat}(\mathbf{Q}^{(1)}) \\
        \mathtt{flat}(\mathbf{V}^{(1)}) \\
        \mathtt{flat}(\mathbf{O}^{(1)}) \\
        \ldots \\
        \mathtt{flat}(\mathbf{K}^{(\numHeads)}) \\
        \mathtt{flat}(\mathbf{Q}^{(\numHeads)}) \\
        \mathtt{flat}(\mathbf{V}^{(\numHeads)}) \\
        \mathtt{flat}(\mathbf{O}^{(\numHeads)}) \\
    \end{bmatrix}
    \in \paramSet{\arch}
\end{align}
where the total number of weights is
\begin{equation}
\size{\arch} = \embedDim(5 + 2|\alphabet| + 2 \mlpDim + 4\sum_{k=1}^{\numHeads}\headDim_k) + \numHeads(2\radius + 1).
\end{equation}
We denote the set of valid architecture-weight pairs as
\begin{equation}
    \netSet \;=\;
    \left\{\,
    (\arch, \param) \;:\; \arch \in \archSet,\, \param \in \paramSet{\arch}
    \,\right\}.
\end{equation}
We denote the architectures without attention (i.e. $\numHeads = 0$) as $\archSet_\mathtt{MLP} = \NN_{>0}^2$. We denote the corresponding valid architecture-set pairs as $\netSet_\mathtt{MLP}$.
\begin{definition}[CPLT]
We refer to a pair $(\arch, \param) \in \netSet$ as a \emph{Constant-Precision Looped Transformer (CPLT)}. When attention is not used (i.e. $\numHeads = 0$), we also refer to the network as a \emph{Constant-Precision Looped MLP}.
\end{definition}

The weight norm of a CPLT $\net \in \netSet$ is the norm of $\param$:
\begin{equation}
    \|\net\|_p \;=\; \|\param\|_p.
\end{equation}
The precision $\precision$, maximum weight value $\maxVal,$ alphabet $\alphabet$, and local positional bias radius $\radius$ are assumed constant for all networks.

\newcommand{\cplt}{\mathtt{CPLT}}
\begin{definition}[$\cplt$ function]
    We denote the realization function of a (halting) CPLT $\net \in \netSet$ as
    \begin{equation}
    \cplt_\pair : \alphabet^*\rightarrow\alphabet^*.
    \end{equation}
    That is, if network $\net \in \netSet$, when given input $x \in \alphabet^*$, outputs a string $y \in \alphabet^*$, then $\cplt_\pair(x) = y$.
\end{definition}

\section{Programs from Weights}
\label{sec:cplt-to-program}

We begin with the lower bound on weight norm in terms of program length. This is achieved by reducing the CPLT to a short program, i.e. constructing a compressed program that simulates the network.

\begin{lemma}[Sparsification]
    \label{prop:sparsification}
    For every CPLT $\net \in \netSet$, there is a computationally equivalent CPLT $(\arch', \param') \in \netSet$, i.e. $\cplt_\net = \cplt_{(\arch', \param')}$, satisfying
    \begin{align}
        \|\arch'\|_1 \le \|\param\|_0 + 2, &&
        \|\param'\|_p \le \|\param\|_p.
    \end{align}
\end{lemma}
\begin{proof}
The number of non-zero weights is $\| \param \|_0$. Therefore, there can be at most $\|\param\|_0$ neurons with some incoming non-zero weights. We prune all useless neurons, including those of attention heads, that have only zero weights to obtain an equivalent network. We also prune attention heads without any useful neurons. Note that the resulting network might have attention heads with varying sizes, but the sum of their sizes is at most $\|\param\|_0$. We maintain at least one embedding neuron and one MLP neuron, giving the $+2$ constant.
\end{proof}

\begin{proposition}[CPLT $\to$ Program]
    \label{prop:programs-from-weights}
    Given an optimal universal machine, any CPLT $\net \in \netSet$ can be simulated using a self-delimiting program $\pi$ of length
    \begin{equation}
    |\pi| \;\le\; \frac{\| \param \|_p^p \log_2 \|\param\|_p^p}{\precision^p} + O\left(\|\param\|_p^p \right).
    \end{equation}
\end{proposition}
\begin{proof}
    We assume there is a program of constant size that can simulate arbitrary Constant-Precision Looped Transformers given access to its $\arch$ and $\param$. We prune the unused neurons (\cref{prop:sparsification}) and we store the transformed weights in a sparse format. There are $\| \param \|_0 \le \| \param \|_p^p/\delta^p$ non-zero weights that we encode with $\log_2 \|\param\|_0$ addressing bits.

    \emph{Part I, Encoding Sizes.} We can encode resulting $\embedDim$, $\mlpDim$, and $\numHeads$ using $O(\log \|\param\|_0)$ bits. To encode the new sizes of the attention heads, we use the fact that a positive integer $n$ has $2^{n-1}$ ordered partitions, so one such partition can be encoded with $n$ bits \citep{cover1973enumerative}.

    \emph{Part II, Encoding Weights.} We concatenate the corresponding weight matrices for all attention heads into a single matrix, that is, one key matrix, one query matrix, etc. The resulting matrices, as well as the non-attention weight matrices, will have size at most $\|\param\|_0 \times \|\param\|_0$ and will contain at most $\|\param\|_0$ non-zero entries. We can encode the locations of the non-zero entries in each matrix using the following scheme. For each matrix, we encode the number of non-zero entries in each row with $O(\| \param \|_0)$ bits using the partition encoding scheme. Then, for each row, we only encode the column locations of the non-zero weights, which takes $\log_2 \|\param\|_0$ bits per non-zero weight. Finally, we encode the values of the non-zero weights using unary representation, which takes another $O(\|\param\|_p^p)$ bits in total.

    The $O(\|\param\|_p^p)$ term also absorbs the size of the program simulating the Transformer, as well as the cost for making the program self-delimiting.
\end{proof}

\section{Weights from Programs}
\label{sec:program-to-cplt}

To prove the upper bound on the weight norm, we show how to construct a CPLT that can simulate a program $\pi$ using weight norm $\|\param\|_p^p \sim \precision^p|\pi| / \log_2|\pi|$. We first show, in \cref{sec:encoding-strings}, how to encode arbitrary binary strings $s \in \{0,1\}^*$ efficiently. Then, in \cref{sec:encoding-programs}, we use this construction to encode the programs we want to execute.

\subsection{Encoding Strings in Weights}
\label{sec:encoding-strings}

We begin by showing how to construct a network that outputs the encoding of a binary string $s \in \{0,1\}^*$ using a small weight norm.

\begin{proposition}[String $\to$ CPLT]
\label{prop:mlp-string}
    There is an injective encoding of binary strings $E : \{0,1\}^* \to \alphabet^*$ such that, for every binary string $s \in \{0,1\}^*$, there is a Constant-Precision Looped MLP $\net \in \netSet_\mathtt{MLP}$ with $\cplt_\net(\varepsilon) = E(s)$ and
    \begin{align}
    \|\param\|_p^p \;\le\; (1+o(1))\frac{\precision^p \, |s|}{\log_2|s|} \
    &&
    \left(|s| \to \infty \right).
    \end{align}
\end{proposition}

\newcommand{\vx}{\mathbf{x}}
\newcommand{\vr}{\mathbf{r}}
\newcommand{\vc}{\mathbf{c}}
\newcommand{\ve}{\mathbf{e}}
\newcommand{\vw}{\mathbf{w}}
\newcommand{\mw}{\mathbf{W}}
\newcommand{\blockSize}{b}
\newcommand{\numOnes}{k}
\newcommand{\logarithm}{L}

\begin{proof}
    For simplicity, we will describe a length-dependent encoding $E_n : \{0,1\}^n \to \alphabet^*$, assuming strings have a fixed length $|s|=n$. Then, we show how to construct a network $\net \in \netSet$ that outputs $E_n(s)$ with
    \begin{equation}
        \|\param\|_p^p \;=\; \precision^p n/\log_2 n + o(n/\log n).
    \end{equation}
    Our construction does not use any attention; all the processing takes place at one position in the context. The key idea is to encode $s$ in a binary matrix $\mw \in \{0,\precision\}^{\blockSize \times \blockSize}$ of size $\blockSize$ and having exactly $k = \|\mw\|_0$ non-zero entries. We then construct a network with a weight matrix storing $\mw$ as a block, plus a small circuit reading the block.

    Assuming $n$ is sufficiently large, we may choose
    \begin{align}
        \blockSize = \left\lceil \frac{n}{(\log_2n)^2} \right\rceil,
        &&
        \numOnes = \left \lceil \frac{n}{\log_2(b^2/n)} \right \rceil.
    \end{align}

    The number of such binary matrices is
    \begin{equation}
        \dbinom{b^2}{k} \;\ge\; \left( \frac{b^2}{k} \right)^k
    \end{equation}
    Using, for example, an enumerative encoding \citep{cover1973enumerative}, such matrices can encode a number of bits equal to
    \begin{equation}
        \log_2 \dbinom{b^2}{k} \;\ge\; k \log_2 \left( \frac{b^2}{k} \right)  \;\ge\; n.
    \end{equation}

    Encoding the matrix into weights will require a norm $\|\mw\|_p^p = \precision^p \numOnes$, where
    \begin{align}
        \numOnes
         \;\sim\; \left \lceil \frac{n}{\log_2 n - 4\log_2\log_2n} \right \rceil
         \;\sim\; \frac{n}{\log_2n}
    \end{align}
    as $n \to \infty$.

    We embed $\mw$ as a block of the MLP input weights matrix, i.e. $\mw_\text{in}$. A simple reader can be implemented using $O(b)$ weights, as follows. We use $2b$ neurons of the residual stream to store two one-hot counters, $\vr,\vc \in \{0,1\}^b$, for the current row and column, respectively. We apply $\mw$ directly on $\vc$, selecting the current column. We select the current row by adding $\vr$, subtracting a bias, and applying the ReLU non-linearity. Formally, let
    \begin{equation}
        z \;=\; \mathbf{1}_b^\intercal \; \text{ReLU}(\mw \vc + \precision \vr - \precision\mathbf{1}_b),
    \end{equation}
    where $\mathbf{1}_b$ is the all ones vector of size $b$. If $\vr = \ve_i$ and $\vc=\ve_j$, then $z = \mw_{ij}$. The circuit for incrementing counters uses another $O(b)$ weights. We describe a full construction in \cref{app:weight-reader}.

    To obtain an encoding that doesn't depend on length, we encode $n$ using $O(\log n)$ symbols and prepend it to $E_n(s)$. This can be implemented in a looped network by an additional $O(\log n)$ weights that emit $n$ before activating the above construction.
\end{proof}

\begin{remark}
    Certain looped neural architectures, such as Transformers with Chain-of-Thought, convert the residual states into alphabet symbols after every loop. In the presence of this kind of information bottleneck, we may, for example, use small circuits \citep{sergeev2026complexity} to convert the column and row pointers into binary.
\end{remark}

\subsection{Encoding Programs in Weights}
\label{sec:encoding-programs}

Having established that looped networks may encode arbitrary strings efficiently, we will use this fact to encode programs and execute them. This is enabled by the universality of the looped architecture:

\begin{proposition}[Universality]
\label{prop:universal-transformer}
    For any Turing Machine, there is a CPLT that simulates it exactly, performing the same computable function.
\end{proposition}
\begin{proof}
    The CPLT may emulate a Turing Machine (TM) using the context as the TM tape. One neuron can be used to mark the current position of the TM head. At that position, another set of neurons can be used to store the state of the TM. Two attention heads can use the relative positions to retrieve the state of the left and right neighbors. The MLP can determine the correct transition and move the head accordingly. Since the context grows with one token per step, the same rate at which the simulated TM head moves through the context, the infinite TM tape is simulated effectively. The emitted output may be sent to the first position using global attention. We give a precise construction in \cref{app:universality}.
\end{proof}

\begin{proposition}[Program $\to$ CPLT]
    \label{prop:weights-from-programs}
    Given an optimal universal machine $U$, any program $\pi$ can be simulated by a CPLT $\net \in \netSet$ with
    \begin{align}
    \|\param\|^p_p \;\le\; (1+o(1))\frac{\precision^p \, |\pi|}{\log_2 |\pi|}
    &&
    (|\pi| \to \infty).
    \end{align}
\end{proposition}
\begin{proof}
    By \cref{prop:universal-transformer}, there is a Constant-Precision Looped Transformer that can emulate $U$ on arbitrary programs. We encode the program $\pi$ as string into the weights with a similar construction as used in \cref{prop:mlp-string}, but instead of emitting it, we simply store it in the context. Then, by simulating $U$ with the input and program in context, we reproduce the behavior of $\pi$. We describe the composition in detail in \cref{app:weights-from-programs}.
\end{proof}

\section{Equivalence of Neural Complexity and Kolmogorov Complexity}
\label{sec:neural-kolmogorov-equiv}

In this section, we state our main result. We begin by defining our measure of neural complexity, and then we relate it to the Kolmogorov complexity.

\begin{definition}[Neural Complexity]
    We define the $\ell_p$ neural complexity, for $p \ge 1$, of a string $s \in \alphabet^*$ as the smallest $\ell_p$ weight norm, raised to the power $p$, of a CPLT emitting $s$:
    \begin{equation}
        \begin{aligned}
        \mathcal{N}_p^p(s)
        \;=\; \min_{\net \in \netSet}
        \left\{
        \|\param\|_p^p
        \;:\;
        \mathtt{CPLT}_{\pair}(\varepsilon) = s
        \right\}.
        \end{aligned}
    \end{equation}

    Similarly, we define the $\ell_p$ neural complexity of a total computable function $f : \alphabet^* \rightarrow \alphabet^*$ as the smallest $\ell_p$ weight norm, raised to the power $p$, of a CPLT computing $f$:
    \begin{equation}
        \begin{aligned}
            \mathcal{N}_p^p(f)
            \;=\;& \min_{\net \in \netSet}
            \left\{
            \|\param\|_p^p
            \;:\;
            \forall s.\mathtt{CPLT}_{\pair}(s) = f(s)
            \right\}.
        \end{aligned}
    \end{equation}
\end{definition}

\begin{formalmain}
    \label{thm:neural-kolmogorov-equiv}
    The neural complexity of a string $s \in \alphabet^*$ is asymptotically equivalent to the Kolmogorov complexity divided by its logarithm. That is,
    \begin{align}
        \mathcal{N}_p^p(s) \;\sim\; \precision^{p} \frac{K(s)}{\log_2 K(s)}
        &&
        (K(s) \to \infty),
    \end{align}
    where $a\sim b$ means that $a/b \to 1$.

    The same equivalence also holds in functional form, for a computable function $f : \alphabet^* \rightarrow\alphabet^*$,
    \begin{align}
        \mathcal{N}_p^p(f) \;\sim\; \precision^{p} \frac{K(f)}{\log_2 K(f)}
        &&
        (K(f) \to \infty),
    \end{align}
\end{formalmain}
\begin{proof}
    The equivalence follows directly from combining \cref{prop:programs-from-weights} (lower bound) and \cref{prop:weights-from-programs} (upper bound).
\end{proof}

\section{Generalization Bounds}
\label{sec:generalization-bounds}

Using the fact that every CPLT $\net$ has a description of length $L = \precision^{-p}\, \|\param\|_p^p \log_2 \|\param\|_p^p+ O(\|\param\|_p^p)$, as implied by \cref{prop:programs-from-weights}, we can derive generalization bounds for a CPLT in terms of its weight norm, independent of its size.

Let $\ell : \alphabet^* \times \alphabet^* \rightarrow[0,1]$ be a bounded loss function over strings.
We define the population loss $\mathcal{L}_\mathcal{D}$ of a CPLT $\net$ over a data distribution $\mathcal{D}$ as
\begin{equation}
    \loss_\mathcal{D}(\pair)
    \;=\;
    \mathbb{E}_{(x,y)\sim \mathcal{D}}\left[ \ell(y, \cplt_{\pair}(x)) \right]
\end{equation}
and the empirical loss $\mathcal{L}_S$ over $n$ pairs of input-output string pairs $S \in (\alphabet^* \times \alphabet^*)^n$ as
\begin{equation}
    \loss_S(\pair)
    \;=\;
    \frac{1}{n}\sum
    _{(x,y) \,\in\, S} \ell(y, \mathtt{CPLT}_{\pair}(x)).
\end{equation}

\newcommand{\confidence}{\epsilon}

\begin{proposition}
    There is a constant $C > 0$, such that, for every data distribution $\mathcal{D} \in \Delta(\alphabet^* \times \alphabet^*)$, every confidence parameter $\confidence > 0$, and every sample size $n$, with probability $1-\confidence$ over samples $S \sim \mathcal{D}^n$, the inequality
\begin{equation}
    \loss_{\mathcal{D}} < \loss_S\ + \sqrt{\frac{\precision^{-p}\,\|\param\|_p^p \log_2 \|\param\|_p^p + C\|\param\|_p^p - \ln(\confidence)}{n}  }
\end{equation}
holds for every CPLT $\net$ that halts on every input.
\end{proposition}
\begin{proof}
    See, for example, Theorem 7.7 in \citet{shalev-shwartz2014understanding}.
\end{proof}

Constraining the above bound to CPLTs that perfectly fit the training data, we get the following:

\begin{proposition}
    There is a constant $C > 0$, such that, for every data distribution $\mathcal{D} \in \Delta(\alphabet^* \times \alphabet^*)$, every confidence parameter $\confidence > 0$, and every sample size $n$, with probability $1-\confidence$ over samples $S \sim \mathcal{D}^n$, the inequality
\begin{equation}
    \mathbb{P}_{\mathcal{D}}\left(\hat y \ne y \right)
    <  \frac{\precision^{-p}\,\|\param\|_p^p \log_2 \|\param\|_p^p + C\|\param\|_p^p - \ln(\confidence)}{n},
\end{equation}
where $\hat y = \cplt_\net(x)$, holds for every CPLT $\net$ that perfectly fits the training samples, i.e. $\cplt_\net(x)=y, \forall (x,y) \in S$, and halts on every other input.
\end{proposition}

\section{Norm Minimization as Compression}
\label{sec:compression}

The equivalence of neural complexity and algorithmic complexity, as stated in \cref{thm:main}, implies that norm minimization is a form of compression. In fact, norm minimization is an asymptotically optimal compressor. We state this formally below.

\begin{proposition}
    Let $(S_n)_n$ be a sequence of training sets for tasks of increasing complexity. Define the negative log likelihood loss, regularized loss, and global optimum:
    \begin{align}
        \ell_n(f) &\;=\; -\log_2 \mathbb{P}(S_n \mid f), \\
        \loss_n(f) &\;=\; \ell_n(f) + \lambda_n \mathcal{N}_p^p(f), \\
        \hat f_n &\;=\; \arg \min \loss_n(f).
    \end{align}
    Define the negative log posterior under the algorithmic prior $\mathbb{P}(f) \propto 2^{-K(f)}$ and its MAP estimate (or, in other words, the two-part code length and the optimal model):
    \begin{align}
        C_n(f) &\;=\; \ell_n(f) + K(f), \\
        f^* _n &\;=\; \arg \min C_n(f).
    \end{align}
    In the limit of infinite task complexity, $K(f^*_n) \to \infty$, and assuming that $\lambda_n$ is calibrated to the task complexity
    \begin{align}
        \lambda_n \;\sim\;
            \delta^{-p}
            \log_2 C_n(f_n^*),
            &&
            (n \to \infty),
    \end{align}
    the minimizer of the regularized loss achieves asymptotically optimal compression:
    \begin{align}
       C_n(\hat f_n) \;\sim\; C_n(f^*_n),
            &&
            (n \to \infty).
    \end{align}
    We give the proof in \cref{app:map}.
\end{proposition}

\section{Weight Priors and Universal Priors}
\label{sec:priors}

Does the norm-complexity equivalence extend to the full prior distribution? In particular, are weight priors that promote small weights, such as the Gaussian prior, equivalent to program priors that promote small programs, such as Solomonoff's universal prior? This question was investigated by previous work for non-looped networks \citep{mingard2025deep} with positive findings. In this section, we investigate briefly our looped Turing-universal architecture with negative (partial) results.

To maintain the universality of CPLTs, we look at priors over networks of unbounded sizes. In particular, we focus our study on weight priors in the limit of infinite width networks, which are representative of modern over-parametrized nets. If weights are sampled i.i.d., then any distribution will produce either (1) almost no networks with small norms, or (2) almost only degenerate networks that never output and never halt.

\begin{proposition}
\label{prop:degenrate-priors}
    For a sequence of weight distributions and increasing network architectures $(\mathcal{D}^{(n)}, \arch^{(n)})_{n\in\NN}$, where $\mathcal{D}^{(n)}\in\Delta(\weightSet)$, $\arch^{(n)} \in \archSet$, and $\size{\arch^{(n)}} \rightarrow\infty$, we generate CPLTs of size $\arch^{(n)}$ by sampling their weights i.i.d. $\param_i \sim \mathcal{D}^{(n)}$. Assume the following limit exists (possibly infinite):
    \begin{equation}
        \mathbb{E}_{\param_i \sim \mathcal{D}^{(n)}}\left[ \|\param\|_0\right]
        \;\rightarrow\;
        \mu.
    \end{equation}
    \begin{enumerate}
        \item If $\mu = \infty$, the probability of sampling a network with small weight norm vanishes. That is, for any $p \ge 1$ and any $c \in \RR$,
        \begin{equation}
        \mathbb{P}_{\param_i \sim \mathcal{D}^{(n)}}\left(\| \param\|_p^p < c \right) \rightarrow 0.
        \end{equation}
        \item If $\mu < \infty$, then almost all networks sampled will never emit any output, regardless of input. That is,
        \begin{align}
            \mathbb{P}_{\param_i \sim \mathcal{D}^{(n)}}\left(\exists s \in \alphabet^*,\
            \cplt_{\arch^{(n)},\param}(s) \text{ emits} \right) \rightarrow 0.
        \end{align}
    \end{enumerate}
    We give the proof in \cref{app:degenerate-priors}.
\end{proposition}

However, in any non-degenerate prior over finite programs, there must be some probability placed on finite non-degenerate programs. Thus, an equivalence of priors does not easily follow from the equivalence of complexities.

Another important issue is that an exponential prior induced by neural complexity, i.e. $\mathbb{P}(f) = \beta^{-\mathcal{N}_p^p(f)}$, is not normalizable, regardless of the chosen $\beta$, due to the logarithmic factor in $\mathcal{N}_p^p(f)$.

\section{Discussion}
\label{sec:discussion}

We have established that minimizing the weight norm of looped neural networks is equivalent to minimizing their algorithmic complexity. This could explain why models generalize so well, while remaining capable of memorizing large amounts of random data. Most importantly, our result only holds for Turing-universal neural architectures, which are currently gaining adoption \citep{zhu2026scaling}. For non-universal neural architectures we refer to prior works \citep{jacot2025deep,mingard2025characterising,ahvonen2025descriptive}.

We have motivated norm minimization based on the ubiquitous weight decay regularization technique. However, several other aspects of deep learning induce an implicit bias towards lower $\ell_p$ norms, such as stochasticity \citep{nacson2019stochastic}, homogeneity \citep{lyu2020gradient, tsilivis2026flavors}, optimization geometry \citep{zhang2024implicit, gunasekar2020characterizing}, initialization \citep{woodworth2020kernel, gunasekar2019implicit}, input noise \citep{bishop1995training}, feature noise \citep{NIPS2013_38db3aed}, or early stopping \citep{ali2019continuoustime}. There are also implicit biases other than norm minimization \citep{NEURIPS2020_f21e255f, JMLR:v25:24-0220, gatmiry2024simplicity}, but they might still be related to complexity measures \citep{valle-perez2019deep, mingard2021sgd, mingard2025deep}.

We have stated our theorem in terms of a depth-recurrent Transformer with a growing context, but we expect it to hold for most looped neural architectures with constant-precision weights, provided they are Turing-universal. The lower bound holds as long as we may prune neurons without incoming weights, which ensures a logarithmic addressing cost. The upper bound only requires a looped MLP and a Turing-complete model.

Regarding weight priors, we have not explored the priors over programs induced by priors over weights when the expected norm is infinite (type 1 in \cref{prop:degenrate-priors}). It remains an open question whether such weight priors might actually induce non-degenerate priors over programs. We also leave it for future work the study of distribution over network dimensions $\mathcal{D}^\arch \in\Delta(\archSet)$, as these are more representative of networks that are capacity-constrained rather than norm-constrained.

\section*{Limitations \& Future Work} Our results are limited to Turing-complete looped networks with constant-precision weights. While such architectures are currently growing in adoption, it could be of interest to study other formulations, such as networks with arbitrary precision, memory constraints, or time constraints. Moreover, this work is focused on the $\ell_p$ norm, motivated by weight decay, but other implicit biases of deep learning might also relate to complexity measures. Even more importantly, our main results are only about the global optima, without discussing the training dynamics. Finally, much work remains to be done to understand the connection between weight priors and algorithmic priors.

\newpage
\section*{AI Use Statement} During the development of this paper, LLM-based tools have been used extensively to identify relevant prior works and construct proofs. The main text was written entirely without any AI assistance, with every word typed by the authors. AI has assisted the authors only in writing the appendices. The authors paid careful attention to check every idea or construction proposed by AI. We have also used AI as a `reviewer' to identify possible issues and areas for improvement.

\section*{Impact Statement}

This paper presents work whose goal is to advance the field of Machine
Learning. There are many potential societal consequences of our work, none which we feel must be specifically highlighted here.

\bibliography{example_paper}
\bibliographystyle{icml2026}

\newpage
\appendix
\crefname{appendix}{Appendix}{Appendices}
\Crefname{appendix}{Appendix}{Appendices}
\crefalias{section}{appendix}
\onecolumn

\section{A looped MLP that reads its own weights}
\label{app:weight-reader}

We construct a precise looped MLP that can read a matrix embedded in its weights using a small reader circuit.

Following Proposition~\ref{prop:mlp-string}, let $b\geq 1$ and
$\mathbf W\in\{0,\delta\}^{b\times b}$, and write
$k=\|\mathbf W\|_0$. Choose two distinct symbols of $\Sigma$, denoted
$\mathtt 0$ and $\mathtt 1$. We construct a looped MLP that contains
$\mathbf W$ as a block of $\mathbf W_{\mathrm{in}}$, emits
\begin{equation}
  \frac{W_{11}}{\delta}\,\cdots\,
  \frac{W_{1b}}{\delta}\,
  \frac{W_{21}}{\delta}\,\cdots\,
  \frac{W_{bb}}{\delta},
  \label{eq:reader-output-order}
\end{equation}
and halts immediately after the last symbol. Apart from the stored block,
only $O(b)$ nonzero parameters are needed. The construction follows the
proof sketch: two one-hot pointers select a matrix entry; the column
pointer advances every step, and its wraparound advances the row pointer.
Reading and advancing take place in the same MLP application.

We assume the constant-precision allows $1\in\mathcal Q$ throughout.

\subsection{State and explicit weights}
\label{app:reader-unit-weights}

Write $\sigma(v)=\max\{v,0\}$. We set $n_h=0$ and work only at the
position initially occupied by $\mathtt{BOS}$. Its residual state is
\begin{equation}
  \mathbf x=(u,r_1,\ldots,r_b,c_1,\ldots,c_b,z,d)
  \in\mathbb R^{2b+3}.
\end{equation}
Here $u$ is a constant coordinate (used as a bias), $\mathbf r$ and $\mathbf c$ are the
row and column pointers, $z$ holds the most recently read entry, and $d$
is a completion flag. Set
\begin{equation}
  \mathbf x^{(0)}=(1,\mathbf e_1,\mathbf e_1,0,0)
  \label{eq:reader-initial-state}
\end{equation}
using the embedding of $\mathtt{BOS}$, and set every other embedding to
zero. Thus, the embedding matrix has three nonzero entries. No hidden
unit writes to $u$: we keep $u=1$ throughout and use it as a bias. Since the MLP has no additive biases, every other context position remains zero throughout.

The loop update follows
\begin{equation}
  \mathbf x^{(t+1)}=\mathbf x^{(t)} +
  \mathbf W_{\mathrm{out}}\sigma(\mathbf W_{\mathrm{in}} \, \mathbf x^{(t)}).
  \label{eq:reader-residual-form}
\end{equation}
We specify both $\mw_{\mathrm{in}}$ and  $\mw_{\mathrm{out}}$ using \cref{tab:reader-weights}.
The rows of $\mathbf W_{\mathrm{in}}$ correspond to the coefficients inside a hidden unit's ReLU activation. The columns of
$\mathbf W_{\mathrm{out}}$ correspond to the listed output weights. All unlisted weights are zero. In the output
column only, use the conventions $c_{b+1}=c_1$ and $r_{b+1}=d$.

\begin{table}[htbp]
\caption{Complete hidden-layer weight description. A notation such as
$z:\,+1$ means an output weight of $+1$ from that hidden unit to residual
coordinate $z$. We use the conventions $c_{b+1}=c_1$ and $r_{b+1}=d$.}
\label{tab:reader-weights}
\centering
\renewcommand{\arraystretch}{1.35}
\begin{tabular}{lllll}
\toprule
Hidden Unit Role & Num. units & Indexing & Activation & Output weights \\
\midrule
Read $\mw$ & $b$ & $i=1,\ldots,b$
 & $a_i=\sigma\!\left(\sum_{j=1}^b W_{ij}c_j+\delta r_i-\delta u\right)$
 & $z:\,+1$ \\
Advance column & $b$ & $j=1,\ldots,b$
 & $v_j=\sigma(c_j)$
 & $c_j:\,-1,\quad c_{j+1}:\,+1$ \\
Advance row & $b$ & $i=1,\ldots,b$
 & $g_i=\sigma(r_i+c_b-u)$
 & $r_i:\,-1,\quad r_{i+1}:\,+1$ \\
Clear old entry & $1$ & ---
 & $q=\sigma(z)$
 & $z:\,-1$ \\
\bottomrule
\end{tabular}
\end{table}

There are $3b+1$ hidden units. In particular,
\begin{equation}
  (\mathbf W_{\mathrm{in}})_{a_i,c_j}=W_{ij},
\end{equation}
so $\mathbf W$ is exactly on the $b\times b$ submatrix of $\mw_\textrm{in}$ connecting the `read $\mw$' hidden units to the residual units storing the column pointer. All other weights are independent of $\mathbf W$. The term $-\delta u$ supplies the read threshold without
requiring an architectural bias.

\paragraph{Output weights.}
After each MLP application, set
\begin{equation}
\begin{aligned}
  \mathtt{emit}&=u,
  &\mathtt{halt}&=d,\\
  \mathtt{score}_{\mathtt 0}&=\delta u-z,
  &\mathtt{score}_{\mathtt 1}&=z.
\end{aligned}
\label{eq:reader-readout}
\end{equation}
All other symbol scores are zero. These formulas specify
$\mathbf w_{\mathtt{emit}}$, $\mathbf w_{\mathtt{halt}}$, and the
unembedding matrix, using five nonzero weights in total. When
$z\in\{0,\delta\}$, the intended symbol has score $\delta$ and every
other symbol has score zero, so no tie-breaking convention is needed.

\subsection{Why the reader works}

Suppose the pointers before an application are
$\mathbf r=\mathbf e_i$ and $\mathbf c=\mathbf e_j$. The read units satisfy
\begin{equation}
  a_\ell=
  \begin{cases}
    W_{ij},&\ell=i,\\
    \sigma(W_{\ell j}-\delta)=0,&\ell\ne i.
  \end{cases}
\end{equation}
Because $z\geq0$, the clearing unit cancels the old value of $z$ through
the residual connection. Consequently,
\begin{equation}
  z^+=z+\sum_{\ell=1}^b a_\ell-\sigma(z)
      =W_{ij}.
\end{equation}
Thus the network reads one entry using exactly the row-selection gate
from the proof sketch.

The `advance column' units simply rotate the one-hot column pointer:
\begin{equation}
  c_1^+=c_b,\qquad c_j^+=c_{j-1}\quad(2\leq j\leq b).
\end{equation}
The `advance row' units implement the Boolean conjunction of ``this is the
current row'' and ``this is the last column'':
\begin{equation}
  g_\ell=\sigma(r_\ell+c_b-1)=
  \begin{cases}
    r_\ell,&c_b=1,\\
    0,&c_b=0.
  \end{cases}
\end{equation}
Hence, with the convention that $g_0=0$, the remaining updates are
\begin{equation}
  r_\ell^+=r_\ell-g_\ell+g_{\ell-1}
  \quad(1\leq\ell\leq b),
  \qquad d^+=d+g_b.
\end{equation}
Before the last column, the row pointer stays put. At the last column,
it moves to the next row. On leaving the last row, it moves into the
completion flag instead. All of these operations use the state
\emph{before} the application; no hidden unit depends on another hidden
unit's output.

Inductively, for $0\leq t<b^2$, the pointers before application $t+1$ are
\begin{equation}
  \mathbf r^{(t)}=\mathbf e_{\,1+\lfloor t/b\rfloor},
  \qquad
  \mathbf c^{(t)}=\mathbf e_{\,1+(t\bmod b)}.
\end{equation}
After that application,
\begin{equation}
  z^{(t+1)}=W_{\,1+\lfloor t/b\rfloor,\,1+(t\bmod b)},
  \qquad
  d^{(t+1)}=\mathbf 1\{t=b^2-1\}.
\end{equation}
The readout therefore emits exactly the sequence in
\eqref{eq:reader-output-order}, one symbol per application, and halts
on application $b^2$. The last symbol is emitted on the halting step,
as required by the model's output convention. The initialization itself
does not emit a symbol.

\paragraph{Parameter cost.}
The read units use $k+3b$ nonzero weights, the `advance column' units at most
$3b$, the `advance row' units $5b$, and the clearing unit two. Including the
three embedding weights and five output weights gives
\begin{equation}
  d_{\mathrm e}=2b+3,\qquad d_{\mathrm{mlp}}=3b+1,\qquad
  \|\param\|_0\leq k+11b+10.
\end{equation}
Every weight belongs to $\{0,\pm1,\pm\delta\}$. More precisely, for
$p>0$,
\begin{equation}
  \delta^p k
  \;\leq\;\|\param\|_p^p
  \;\leq\;\delta^p(k+2b+1)+9b+9
  \;=\;\delta^p k+O_{\delta,p}(b).
  \label{eq:reader-norm-cost}
\end{equation}
With the choice of $b$ in Proposition~\ref{prop:mlp-string}, this overhead is
$o(n/\log n)$, so it does not change the leading norm cost of the
stored matrix.

\newpage
\section{A looped transformer that simulates a Turing machine}
\label{app:universality}

We prove Proposition~\ref{prop:universal-transformer} by constructing, for
any Turing machine, a CPLT with the same partial input--output function.
The construction uses three attention heads and weights in
$\{-1,0,1\}\subseteq\mathcal Q$.

\paragraph{Machine convention.}
Use a one-sided tape with alphabet $\Gamma\supseteq\Sigma$, states $Q$, and initial
contents $\vdash s\,\square\square\cdots$, where $\square,\vdash\notin\Sigma$
are the blank and left-end marker. The head starts at cell zero in state $q_0$. The left-end marker $\vdash$ is preserved, never written
elsewhere, and never crossed. For each $(q,a)\in Q\times\Gamma$, write
\[
  \tau(q,a)=(q',w,d,o,\eta)
\]
for the transition that writes $w\in\Gamma$, moves by
$d\in\{-1,0,1\}$, enters state $q'$, emits
$o\in\Sigma\cup\{\varepsilon\}$, and then halts if the halt bit $\eta$ is one.
Here $\varepsilon$ means no output. Any Turing machine can be put in
this form by the usual tape simulation and, if needed, a final routine
that emits its tape output. We simulate one transition of this machine
per transformer iteration.

\paragraph{Encoding and initialization.}
Each position stores a constant $u=1$, a one-hot tape symbol
$\mathbf v=(v_a)_{a\in\Gamma}$, a head bit $h$, and state bits
$\mathbf s=(s_q)_{q\in Q}$. Only the head position has a nonzero state
bit: $s_q=1$ for its current state. Two bits $\ell,r$ mark the first
and last context positions. Add three initially zero buffers
$(\mathbf v^k,\mathbf s^k)$ for $k\in\{\mathrm L,\mathrm R,\mathrm G\}$,
and initially zero output registers $(e_x)_{x\in\Sigma},z$.
The embeddings set the following coordinates to one and all others to zero:
\[
\begin{array}{c|l}
  \text{Token} & \text{Coordinates set to one} \\\hline
  \mathtt{BOS} & u,\ell,h,v_{\vdash},s_{q_0} \\
  x\in\Sigma & u,v_x \\
  \mathtt{EOS},\mathtt{EXT} & u,r,v_{\square}
\end{array}
\]
At each iteration, the MLP clears $r$ at every existing position;
appending $\mathtt{EXT}$ then marks the new right boundary.

\paragraph{Attention.}
Use three heads of width $|\Gamma|+|Q|$, with scores
\begin{align*}
  S^{\mathrm L}_{i,j}&=\ell_i\ell_j+\mathbf 1\{i=j-1\}, &
  S^{\mathrm R}_{i,j}&=r_i r_j+\mathbf 1\{i=j+1\}, &
  S^{\mathrm G}_{i,j}&=h_i u_j.
\end{align*}
Only the first query/key rows are nonzero: they select
$(\ell,\ell)$, $(r,r)$, and $(u,h)$, respectively.
The positional biases are one at the displayed offsets and zero
elsewhere. The value and output matrices copy
$(\mathbf v_i,\mathbf s_i)$ from the maximizing source $i$ into the
corresponding buffer at query $j$; all other entries are zero.
Thus $\mathrm L$ and $\mathrm R$ retrieve the neighboring cells, or
the query cell itself when that neighbor is missing.
Head $\mathrm G$ retrieves the unique head cell at every position,
including position zero. All maxima are unique.

\paragraph{The MLP.}
Suppress the position index and write $\sigma=\operatorname{ReLU}$.
For each $(q,a)\in Q\times\Gamma$, create four hidden units:
\begin{align*}
 A_{q,a}&=\sigma(s_q+v_a-u), &
 L_{q,a}&=\sigma(s_q^{\mathrm L}+v_a^{\mathrm L}-u-\ell),\\
 G_{q,a}&=\sigma(s_q^{\mathrm G}+v_a^{\mathrm G}-u), &
 R_{q,a}&=\sigma(s_q^{\mathrm R}+v_a^{\mathrm R}-u-r).
\end{align*}
The unit $A_{q,a}$ detects the current head's state and scanned symbol
locally; $L_{q,a}$ and $R_{q,a}$ detect the same event at a neighbor.
The boundary terms suppress the self-fallbacks. The unit $G_{q,a}$
detects the head's state and symbol at every position.

For $\tau(q,a)=(q',w,d,o,\eta)$, specify the output connections as follows.
Connect $A_{q,a}$ to $v_a$ with weight $-1$ and to $v_w$ with weight
$+1$; these cancel when $w=a$. To place the new head, connect the unit
in the following table to both $h$ and $s_{q'}$ with weight $+1$:
\[
\begin{array}{c|ccc}
  \text{Head movement }d & 0 & +1 & -1 \\\hline
  \text{Unit} & A_{q,a} & L_{q,a} & R_{q,a}
\end{array}
\]
For example, a cell receives a right-moving head from its left neighbor.
Connect $G_{q,a}$ to $e_o$ when $o\ne\varepsilon$, and to $z$ when
$\eta=1$, with weight $+1$.

The coordinates $h,\mathbf s,r$, all buffers, and all output registers
must be overwritten rather than accumulated. For each such coordinate
$c$, add a hidden unit $\sigma(c)$ with output weight $-1$ into $c$.
Since these coordinates are nonnegative, this cancels their residual
values: $c-\sigma(c)=0$. The transition units then supply their new
values. All units are evaluated in parallel, with no nested ReLUs;
$u,\ell$ remain unchanged. These rules specify
$\mathbf W_{\mathrm{in}}$ and $\mathbf W_{\mathrm{out}}$ completely,
with all unspecified weights zero.

\paragraph{Readout and correctness.}
At position zero, choose the readout weights to give
\[
  \mathtt{emit}=\sum_{x\in\Sigma}e_x,
  \qquad \mathtt{score}_x=e_x,
  \qquad \mathtt{halt}=z.
\]
Each iteration writes the new symbol at the old head position, moves
the head and updates its state, and emits and halts as prescribed by
$\tau$. The buffers are cleared, and the appended $\mathtt{EXT}$ supplies
a blank cell and the new right-boundary flag.

There is always enough tape: after $t$ transitions the head position
satisfies $p_t\le t$, whereas $N_t=|s|+2+t$. Hence
$p_t+1\le N_t-1$, so even a right-moving transition lands in an existing
context cell. The embeddings establish the initial configuration, and
the update rules preserve the simulation inductively. The CPLT therefore
halts on exactly the same inputs and produces exactly the same output,
including any symbol emitted on the halting step. Its weights depend
only on the machine, not on the input. Applying the construction to an
optimal universal machine proves Proposition~\ref{prop:universal-transformer}.

\newpage
\section{A looped transformer that executes a stored program}
\label{app:weights-from-programs}

We prove \cref{prop:weights-from-programs} by connecting the reader in
\cref{app:weight-reader} to the simulator in \cref{app:universality}.
The reader first writes the encoded program into the context; a fixed
machine then decodes and executes it.

\paragraph{The stored program.}
Let $n=|\pi|$, and use the encoding $E$ and parameters $b,k$ from
\cref{prop:mlp-string}. Choose its $O(\log n)$-symbol length header to be
self-delimiting. Thus $E$ has a fixed computable inverse: read $n$, then
invert the enumerative encoding of the following $b^2$ bits.
Write $E(\pi)=a_1\cdots a_T$, with $a_t\in\{\mathtt0,\mathtt1\}$.

Keep the reader's coordinates separate from the simulator's, and
initialize them only at $\BOS$. Retain its scalar output
$z_{\mathrm{rd}}\in\{0,\delta\}$ and completion flag $d_{\mathrm{rd}}$,
but remove its external emission and halting weights.
Add a validity bit $\nu$, initially zero, with update
\[
  \nu^+=u_{\mathrm{rd}}-d_{\mathrm{rd}},
\]
where $u_{\mathrm{rd}}$ is the reader's own constant coordinate.
After iteration $t=1,\ldots,T$, we have
$(\nu,z_{\mathrm{rd}})=(1,\delta a_t)$; both coordinates are zero
initially and for $t>T$. The flag $d_{\mathrm{rd}}$ becomes one at
$t=T$ and stays one. This interface adds only $O(1)$ weights.

\paragraph{The fixed interpreter.}
Fix a Turing machine $D$, with initial state $q_0$, that starts on
\[
  \vdash\;s\;\#\;E(\pi)\;\square\square\cdots,
\]
decodes $\pi$, and simulates $U(\pi,s)$. Here $\#$ is a new tape symbol.
The machine $D$ depends on $U$ and $E$, but not on $\pi$ or $s$.
Use its simulator from \cref{app:universality}, adding a state
$q_{\mathrm{wait}}$ that stays put, preserves the tape, and neither
emits nor halts. Initialize the head at $\BOS$ in this waiting state,
and embed $\EOS$ as $\#$ instead of a blank. Otherwise use the same
simulator embeddings.

\paragraph{Writing the program and starting the interpreter.}
Extend the simulator's global attention head to fill three fresh
buffers with
\[
  \beta_1=\delta^{-1}z_{\mathrm{rd}},\qquad
  \beta_0=\nu-\delta^{-1}z_{\mathrm{rd}},\qquad
  \bar d=d_{\mathrm{rd}}.
\]
While loading, the head remains at $\BOS$, so these values reach every
position. The factor $\delta^{-1}$ is implemented by repeating the
value row selecting $z_{\mathrm{rd}}$ exactly $\delta^{-1}$ times and
summing with unit output weights. Thus all weights remain in
$\mathcal Q$, and no stored matrix entry is duplicated.

Using the simulator coordinates $u,r,v_a,s_q$ from
\cref{app:universality}, add the hidden units
\[
  B_a=\operatorname{ReLU}(r+\beta_a-u)
  \quad(a\in\{\mathtt0,\mathtt1\}),
  \qquad
  C=\operatorname{ReLU}(s_{q_{\mathrm{wait}}}+\bar d-u).
\]
Connect $B_a$ to $v_a$ with weight $+1$ and to $v_{\square}$ with
weight $-1$. Connect $C$ to $s_{q_{\mathrm{wait}}}$ with weight $-1$
and to $s_{q_0}$ with weight $+1$. Clear the three new buffers using
the same clearing units as in \cref{app:universality}.
These units operate in parallel with the reader and simulator.

Initially $\beta_0=\beta_1=0$, so the delimiter is untouched.
At iteration $t+1$, for $1\le t\le T$, the reader's $t$-th symbol
is written at the right boundary, which is the fresh $\EXT$ at
position $|s|+1+t$. On iteration $T+1$, the last symbol is written
and $C$ switches the head to $q_0$ simultaneously. The next iteration
therefore starts $D$ on the completed tape displayed above.
Thereafter $\beta_0=\beta_1=0$ and $q_{\mathrm{wait}}$ is never
re-entered, so the interface has no further effect.
Use only the simulator's external readout: loading and decoding emit
nothing, and execution has exactly the output and halting behavior
of $U(\pi,s)$.

\paragraph{Parameter cost.}
The stored matrix is unchanged. The reader and length header cost
$\delta^p k+O(b+\log n)$ by
\cref{eq:reader-norm-cost,prop:mlp-string}; the fixed simulator and
interface add only $O_{U,\delta,p}(1)$. Hence
\[
  \|\theta\|_p^p
  =\delta^p k+O_{U,\delta,p}(b+\log n+1)
  =(1+o(1))\frac{\delta^p n}{\log_2 n},
\]
as required. The same construction works for every input $s$.

\newpage
\section{Weight decay as asymptotically optimal compression}
\label{app:map}

The main theorem relates minimum weight norm to $K/\log K$, whereas
MAP estimation penalizes $K$. We show that calibrating the regularization
coefficient to the optimal total code length compensates for this
logarithmic factor, yielding asymptotically optimal two-part compression.

Throughout, the computational model, reference machine, $p$, and $\delta$
are fixed. All minima range over total computable functions
$f:\Sigma^*\to\Sigma^*$. The likelihood is a probability mass, so its
negative logarithm is nonnegative, with $-\log_2 0:=+\infty$.

\begin{proposition}[Asymptotically optimal compression]
\label{prop:map-compression}
Let $(S_n)_n$ be a sequence of training sets. Define
\[
  \ell_n(f)=-\log_2\mathbb{P}(S_n\mid f),
\]
and the regularized loss and two-part code length by
\begin{align*}
  \loss_n(f)&=\ell_n(f)+\lambda_n\mathcal{N}_p^p(f),
  &\widehat f_n&\in\arg\min_f\loss_n(f),\\
  C_n(f)&=\ell_n(f)+K(f),
  &f_n^*&\in\arg\min_f C_n(f).
\end{align*}
Assume these minima are finite and attained. Under the prior
$\mathbb{P}(f)\propto 2^{-K(f)}$, $C_n(f)$ is the negative log posterior
up to an $f$-independent constant, so $f_n^*$ is a MAP estimate.

If $K(f_n^*)\to\infty$ and the positive regularization coefficients satisfy
\begin{equation}
  \lambda_n\sim\delta^{-p}\log_2 C_n(f_n^*),
  \qquad n\to\infty,
  \label{eq:compression-calibration}
\end{equation}
then
\begin{equation}
  C_n(\widehat f_n)\sim C_n(f_n^*),
  \qquad n\to\infty.
  \label{eq:compression-optimality}
\end{equation}
\end{proposition}

\begin{proof}
Write $B_n=C_n(f_n^*)$ for the optimal code length. Since the loss is
nonnegative, $K(f_n^*)\le B_n$ and $B_n\to\infty$.
We use \cref{thm:main} in the form
\[
  R(f):=\delta^{-p}\mathcal{N}_p^p(f)
  \sim h(K(f)),
  \qquad h(k):=\frac{k}{\log_2 k},
\]
with the asymptotic error uniform as $K(f)\to\infty$.
The proof has three steps: compare the penalties on a bounded complexity
range, show that both minimizers lie in that range, and compare their
objective values.

\paragraph{1. The penalties agree on the relevant range.}
We claim that
\begin{equation}
  e_n:=\sup_{K(f)\le 2B_n}
  \left|\lambda_n\mathcal{N}_p^p(f)-K(f)\right|
  =o(B_n).
  \label{eq:compression-uniform}
\end{equation}
Put $t_n=B_n/\log_2 B_n$. For $t_n\le K(f)\le 2B_n$,
we have $\log_2 K(f)\sim\log_2 B_n$ uniformly. The main theorem and
\eqref{eq:compression-calibration} therefore give, uniformly on this range,
\[
  \lambda_n\mathcal{N}_p^p(f)
  =(1+o(1))K(f)\frac{\log_2 B_n}{\log_2 K(f)}
  =K(f)+o(B_n).
\]
For $K(f)<t_n$, both penalties are negligible relative to $B_n$.
Indeed, the main theorem and finiteness of the set of functions with
bounded Kolmogorov complexity give $R(f)=O(1+h(t_n))$ uniformly.
Consequently,
\[
  K(f)\le t_n=o(B_n),
  \qquad
  \lambda_n\mathcal{N}_p^p(f)
  =O(\log_2 B_n+t_n)=o(B_n).
\]
Together, the two ranges prove \eqref{eq:compression-uniform}.

\paragraph{2. Both minimizers lie in this range.}
We already have $K(f_n^*)\le B_n$. By nonnegativity of the loss,
global optimality, and \eqref{eq:compression-uniform},
\[
  \lambda_n\mathcal{N}_p^p(\widehat f_n)
  \le\loss_n(\widehat f_n)
  \le\loss_n(f_n^*)
  \le B_n+e_n=(1+o(1))B_n.
\]
This forces $K(\widehat f_n)\le 2B_n$ for all sufficiently large $n$.
Otherwise, along a subsequence, the main theorem and monotonicity of $h$
for large arguments would imply
\[
  \lambda_n\mathcal{N}_p^p(\widehat f_n)
  \ge(1-o(1))\log_2 B_n\,
       \frac{2B_n}{\log_2(2B_n)}
  =(2-o(1))B_n,
\]
a contradiction.

\paragraph{3. Compare the optimal code lengths.}
The bound \eqref{eq:compression-uniform} now applies to both minimizers.
Hence
\begin{align*}
  B_n
  &\le C_n(\widehat f_n)\\
  &\le\loss_n(\widehat f_n)+e_n\\
  &\le\loss_n(f_n^*)+e_n\\
  &\le B_n+2e_n.
\end{align*}
Dividing by $B_n$ and using $e_n=o(B_n)$ proves
\eqref{eq:compression-optimality}.
\end{proof}

\paragraph{Interpretation and calibration.}
The conclusion compares total description lengths; it does not require
$\widehat f_n=f_n^*$. Since $\mathcal{N}_p^p(f)$ minimizes over all
representations of $f$, the result concerns global optimization over all
allowed architectures and weights. The calibration is an oracle condition,
but only a logarithmic estimate is needed: $C_n(f_n^*)$ in
\eqref{eq:compression-calibration} may be replaced by any positive
$\widetilde B_n$ satisfying
$\log_2\widetilde B_n\sim\log_2 C_n(f_n^*)$.

\newpage
\section{Degenerate weight priors}
\label{app:degenerate-priors}

\begin{proof}[Proof of Proposition~\ref{prop:degenrate-priors}]
Write
\[
  d_n=\size{\arch^{(n)}},\qquad
  q_n=\mathbb P_{\vartheta\sim\mathcal D^{(n)}}(\vartheta\ne0),
  \qquad
  N_n=\|\param\|_0.
\]
Since the $d_n$ weights are sampled independently,
$N_n\sim\operatorname{Binomial}(d_n,q_n)$. Thus
\[
  \mu_n:=\mathbb E[N_n]=d_nq_n\longrightarrow\mu,
  \qquad
  \operatorname{Var}(N_n)=d_nq_n(1-q_n)\le\mu_n.
\]
All probabilities below refer to the sampled weights at index $n$.

\paragraph{Case 1: $\mu=\infty$.}
Every nonzero weight has magnitude at least $\precision$, so
$\|\param\|_p^p\ge\precision^pN_n$. Fix $p\ge1$ and
$c\in\mathbb R$. If $c\le0$, the desired probability is zero.
Otherwise, for all sufficiently large $n$,
$c\le\precision^p\mu_n/2$, and Chebyshev's inequality gives
\[
  \mathbb P\bigl(\|\param\|_p^p<c\bigr)
  \le \mathbb P(N_n\le\mu_n/2)
  \le \frac{4\operatorname{Var}(N_n)}{\mu_n^2}
  \le \frac{4}{\mu_n}
  \longrightarrow0.
\]

\paragraph{Case 2: $\mu<\infty$.}
We prove the stronger statement that, with probability tending to one,
the network neither emits nor halts on any input. The key observation
is that either action requires a residual coordinate with both a
nonzero incoming weight and a nonzero readout weight.

Fix $n$ and suppress the dependence of the architectural dimensions on
$n$. For each residual coordinate $i\in[\embedDim]$, let
$\mathcal I_i\subseteq[d_n]$ be the indices of the parameters in row
$i$ of
\[
  \embed,\qquad \mathbf W_{\mathrm{out}},\qquad
  \mathbf O^{(1)},\ldots,\mathbf O^{(\numHeads)}.
\]
These are all the learned weights that can initialize or write to
coordinate $i$. In particular,
\[
  |\mathcal I_i|
  =|\alphabetExt|+\mlpDim+\sum_{k=1}^{\numHeads}\headDim_k,
  \qquad
  \sum_{i=1}^{\embedDim}|\mathcal I_i|\le d_n.
\]

If every weight indexed by $\mathcal I_i$ is zero, then coordinate
$i$ is identically zero at every context position throughout every
execution. Indeed, its embedding is zero for every symbol, including
$\EXT$. Row $i$ of each attention output projection and of the MLP
output matrix is also zero, so neither sublayer changes this coordinate.
The residual connections only preserve its zero value. Induction over
the loop steps proves the claim, simultaneously for all inputs.

For $r\in\{\mathtt{emit},\mathtt{halt}\}$, let
$\mathbf w_r$ be the corresponding readout vector and define
\[
  B_{n,r}
  =\bigcup_{i=1}^{\embedDim}
    \left\{(\mathbf w_r)_i\ne0
    \ \text{and}\ \exists j\in\mathcal I_i:\param_j\ne0\right\}.
\]
On $B_{n,r}^{\mathrm c}$, every coordinate with a nonzero readout
weight has all its incoming weights zero, and hence remains zero.
Consequently, $\mathbf w_r^\intercal\mathbf z^{(t)}=0$ at every
step of every execution: the corresponding action never occurs.

The readout parameter $(\mathbf w_r)_i$ is distinct from every
parameter indexed by $\mathcal I_i$. Independence and the union bound
therefore give
\[
  \mathbb P(B_{n,r})
  \le \sum_{i=1}^{\embedDim}\sum_{j\in\mathcal I_i}
     \mathbb P\bigl((\mathbf w_r)_i\ne0,\ \param_j\ne0\bigr)
  =q_n^2\sum_{i=1}^{\embedDim}|\mathcal I_i|
  \le d_nq_n^2.
\]
Applying the union bound to the two readouts yields
\[
  \mathbb P\left(
    \exists s\in\alphabet^*:\quad
    \text{the network emits or halts on input }s
  \right)
  \le \mathbb P(B_{n,\mathtt{emit}}\cup B_{n,\mathtt{halt}})
  \le 2d_nq_n^2
  =\frac{2\mu_n^2}{d_n}
  \longrightarrow0,
\]
since $\mu_n\to\mu<\infty$ and $d_n\to\infty$.
This proves both the absence of emission on every input with high
probability and, in particular, the stated vanishing probability of
halting on every input.
\end{proof}

\end{document}